\documentclass{ecai}

\usepackage{latexsym}

\usepackage{authblk}

\usepackage{times}  
\usepackage{helvet} 
\usepackage{courier}  
\usepackage[hyphens]{url}  
\usepackage{graphicx} 
\usepackage{graphicx}  

\usepackage{csquotes}

\usepackage{times}
\usepackage{graphicx}     
\usepackage{amsmath,amsthm}
\usepackage{txfonts}
\usepackage[ruled,vlined]{algorithm2e}

\usepackage{bbm}

\makeatletter
\newcommand{\newreptheorem}[2]{%
\newenvironment{rep#1}[1]{%
 \def\rep@title{#2 \ref{##1}}%
 \begin{rep@theorem}}%
 {\end{rep@theorem}}}
\makeatother

\newcommand{\pazocal}{\cal}

\newtheorem{THEOREM}{Theorem}
\renewenvironment{theorem}{\begin{THEOREM} }%
                        {\end{THEOREM}}

\newtheorem{LEMMA}[THEOREM]{Lemma}
\newenvironment{lemma}{\begin{LEMMA} \hspace{-.85em} {\bf :} }%
                      {\end{LEMMA}}
\newtheorem{COROLLARY}[THEOREM]{Corollary}
                          {\end{COROLLARY}}
\newtheorem{PROPOSITION}[THEOREM]{Proposition}
\newenvironment{proposition}{\begin{PROPOSITION} \hspace{-.85em} {\bf :} }%
                            {\end{PROPOSITION}}
\newtheorem{DEFINITION}[THEOREM]{Definition}
\newenvironment{definition}{\begin{DEFINITION} \rm}%
                            {\end{DEFINITION}}
\newtheorem{CLAIM}[THEOREM]{Claim}
                            {\end{CLAIM}}
\newtheorem{EXAMPLE}[THEOREM]{Example}
                            {\end{EXAMPLE}}
\newtheorem{REMARK}[THEOREM]{Remark}
                            {\end{REMARK}}
							\newtheorem{NOTATION}[THEOREM]{Notation}
							                            {\end{NOTATION}}
\renewenvironment{proof}{\noindent {\bf Proof:} \hspace{.677em}}%
                     {}

\DeclareMathAlphabet{\mathitbf}{OML}{cmm}{b}{it}

\newcommand{\U}{{\bf U}}
\newcommand{\V}{{\bf V}}

\newcommand{\blemma}{\begin{lemma}}

\newcommand{\elemma}{\end{lemma}}

\newcommand{\bthm}{\begin{theorem}}
\newcommand{\ethm}{\end{theorem}}
\newcommand{\bprf}{\begin{proof}}
\newcommand{\eprf}{\end{proof}}
\newcommand{\bpro}{\begin{proposition}}
\newcommand{\epro}{\end{proposition}}
\newcommand{\bi}{\begin{itemize}}
\newcommand{\ei}{\end{itemize}}
\newcommand{\be}{\begin{enumerate}}
\newcommand{\ee}{\end{enumerate}}
\newcommand{\beq}{\begin{equation}}
\newcommand{\eeq}{\end{equation}}
\newcommand{\bcase}{\begin{cases}}
\newcommand{\ecase}{\end{cases}}
\renewcommand{\mathit}{\emph}

\begin{document}

	
\title{On Constraint Definability \\ in  Tractable Probabilistic Models}

\author{Ioannis Papantonis\institute{University of Edinburgh, \\ email:\tt~ i.papantonis@sms.ed.ac.uk} \and Vaishak Belle\institute{University of Edinburgh \& Alan Turing Institute, \\ email:\tt~ vaishak@ed.ac.uk}}

\maketitle

\begin{abstract}
Incorporating constraints is a major concern in probabilistic machine learning. A wide variety of problems require predictions to be integrated with reasoning about constraints, from  modelling routes on maps to approving loan predictions. In the former,  we may require the prediction model to respect the presence of physical paths between the nodes on the map, and in the latter, we may require that the prediction model respect fairness constraints that ensure that outcomes are not subject to bias. Broadly speaking, constraints may be probabilistic, logical or causal, but the overarching challenge is to determine if and how a model can be learnt that handles all the declared constraints. To the best of our knowledge, this is largely an open problem. In this paper, we consider a mathematical inquiry on how the learning of  tractable probabilistic models, such as sum-product networks, is possible while incorporating constraints.
\end{abstract}

\section{Introduction}
Incorporating constraints is a major concern in data mining and probabilistic machine learning \cite{Raedt:2010:CPD:2898607.2898874,inproceedings10,KR148005}. A wide variety of problems require the prediction to be integrated with reasoning about constraints, from modelling routes on maps \cite{Shen2018ConditionalPM,pmlr-v80-xu18h} to approving loan predictions \cite{loan}. That is, when modelling routes, we may require the prediction model to respect the presence of physical paths between nodes on the map, in the sense of disallowing impossible or infeasible paths. Analogously, when approving loans, we may have categorical requirements that loans should not be approved for those with a criminal record, but we may also have conditional constraints for eliminating bias, e.g, the prediction should not penalize the individual based on gender.

Broadly, background information may come in different forms, including independency \cite{Zemel:2013:LFR:3042817.3042973,article3} constraints and logical formulas \cite{KR148005,pmlr-v80-xu18h}, but of course the challenge is if and how we are able to provide (or learn) a model that is able to handle all the declared constraints. To the best of our knowledge, this is largely an open problem, at least in the sense of providing a general solution to a certain class of probabilistic models.

In addition to incorporating prior knowledge as constraints for training a probabilistic model, a second and equally significant way to utilize constraints is in order to enforce a set of properties on the resulting models. For example, historic data on college admissions exhibit a clear bias based on gender or race \cite{Leonard1999,Silverstein2000StandardizedTT}. More generally, there is an abundance of data that reflect historical or cultural biases, prompting the rapid development of the area of fair machine learning \cite{Zemel:2013:LFR:3042817.3042973,article3,Hardt:2016:EOS:3157382.3157469,GrgiHlaa2016TheCF}. Roughly, the idea is to place a constraint (e.g. a formalization that captures, for example, demographic parity \cite{article3} or equality of opportunity \cite{Hardt:2016:EOS:3157382.3157469}) on the predictions of the resulting model so that biased behaviour is not exhibited.

In this paper, we consider a mathematical enquiry on the definability of constraints while training/learning a probabilistic model. Note however that performing inference on probabilistic models is a computationally intractable problem \cite{article50}. This has given rise to probabilistic tractable models \cite{KR148005,6130310} where conditional or marginal distributions can be computed in time linear in the size of the model. Although initially limited to low tree-width models \cite{bach2002thin}, recent tractable models such as sum product networks (SPNs) \cite{SPN_structure_learning,6130310} and probabilistic sentential decision diagrams (PSDDs) \cite{KR148005,liang2017learning} are derived from arithmetic  circuits  (ACs)  and  knowledge  compilation  approaches, more generally \cite{choi2017relaxing,darwiche2002logical}, which exploit efficient function representations and also capture high tree-width models. These models can also be learnt from data \cite{SPN_structure_learning,liang2017learning} which leverage the efficiency of inference. Consider that in classical structure learning approaches for graphical models, once learned, inference would have to be approximated, owing to its intractability. In that regard, such models offer a robust and tractable framework for learning and inferring from data. Owing to these properties and their increasing popularity for a wide range of applications \cite{Choi:2015:TLS:2832581.2832649,Liang2018LearningLC,6130310}, we focus on this class of models in our work.

We are organised as follows. We first review the recent advances in constrained machine learning. Then we briefly review SPNs, and some preliminaries on constrained optimisation. We then turn to our main results. Finally, we conclude with discussions.

\section{Related work and Context}

During the last years, there have been ongoing attempts to address the problem of incorporating constraints during training or in prediction. However, most approaches focus either on logical constraints or probabilistic constraints, but not both, in a bespoke manner. For example, \cite{pmlr-v80-xu18h} examine the problem of imposing certain structure in the outcome of a classification algorithm. They approach this by adding an additional term in the objective function, one accounting for the probability of a state satisfying the given constraint. \cite{MrquezNeila2017ImposingHC} consider the case of training a neural network under some constraints. They create two variants of this problem, one where results from optimization theory are utilized in order to efficiently solve the problem, under hard constraints, as well as a relaxation of this problem, with soft constraints \cite{Gill81,Flet87}, where terms corresponding to the constraints are added into the objective function.

Alternative ways to utilize prior knowledge have been proposed as well, such as \cite{Stewart:2017:LSN:3298483.3298610}. In this work, the authors propose a framework for the semi-supervised training of neural networks. The key insight is that pre-existing knowledge can be used to create a regularizer, prompting the network to satisfy this information. 

Other interesting approaches stem from the field of knowledge representation. Probabilistic sentential decision diagrams (PSDDs) are representations of probability distributions over a propositional theory. One of the advantages of this formalism is that it is straightforward to incorporate logical constraints into the model. Due to this feature, as well as their high performance, PSDDs have been utilized in a wide range of applications, including preference learning \cite{Choi:2015:TLS:2832581.2832649} and modelling route distributions on a map \cite{Shen2018ConditionalPM}. However, up to this point, one of their limitation has been that it is not clear how to incorporate probabilistic assumptions regarding the variables. 

Data mining is an other field that utilizes constraints, this time in order to recover sets of variables satisfying some properties. \cite{Raedt:2010:CPD:2898607.2898874} attempt to develop a structured way to apply constrained programming techniques in pattern mining or rule discovery. A key difference between this method and the ones we mentioned on the previous paragraph, however, is that, in this case, the result consists of the valid assignments, and there is no predictive model.

Introducing constraints has been explored in other settings as well, such as in order to control a model's complexity. \cite{inproceedings10} consider an approach where they constrain the expected value of a quantity, modelled using open-world probabilistic databases. By doing that, they go on to show how this constrain strengthens the semantics of such databases. However, since the new problem is difficult, in the general case, they rely on approximating the corresponding solution. 

Our contribution lies in introducing an approach for training generative models under probabilistic constraints. We borrow concepts from optimization theory and develop a paradigm related to \cite{MrquezNeila2017ImposingHC}. A key difference is that their approach, although similar in spirit, 
takes into account constraints that are expressed in terms of the model's variables. Thus, they correspond to functional relationships that the output variables should respect, so, consequently, they are not of a probabilistic nature. Our approach does allow for probabilistic constraints, as well as it does not demand them to be directly expressed as an equation (or a system of equations). Indeed, in the following sections, we will provide insights about the link between the probabilistic constraints and the system of equations they induce. Interestingly, even when using generative models, where it is easier to specify probabilistic dependencies, new challenges arise. Perhaps the most notable one is coming up with an efficient way to answer conditional or marginal queries (since the constraints will probably involve some of these quantities).

In our proposed framework we suggest to utilize tractable probabilistic models \cite{6130310,KR148005}, where conditional or marginal distributions can be computed in time linear in the size of the model, in order to overcome this challenge. Specifically, we will base our presentation on sum-product networks (SPNs) \cite{6130310}. SPNs are instances of arithmetic circuits (ACs) \cite{choi2017relaxing} that compactly represent the network polynomial \cite{Darwiche:2003:DAI:765568.765570} of a Bayesian network (BN). In the presence of latent variables they can also be seen as a deep architecture with probabilistic semantics \cite{6130310}, leading to numerous extensions, e.g.,  for mixed discrete-continuous domains \cite{inproceedings18}, and applications, including classification \cite{Liang2018LearningLC} and computer vision \cite{6130310}.

In this paper we explore the following: can SPNs (or, generally, tractable models) be used in order to train generative models subject to probabilistic constraints? Furthermore, what kind of constraints on the model's variables are induced by this procedure? We demonstrate how to incorporate various types of probabilistic relationships into the model, using different optimization approaches, specifically targeting hard and soft constraints. Furthermore, we provide a discussion about how our goal relates to other approaches, developed for training classification models.

\section{Background}
In this section we will briefly review SPNs, some causality related concepts, as well as some optimization approaches.\\
\subsection{SPNs}
\begin{figure}[t]
  \centering
    \includegraphics[scale=1.3]{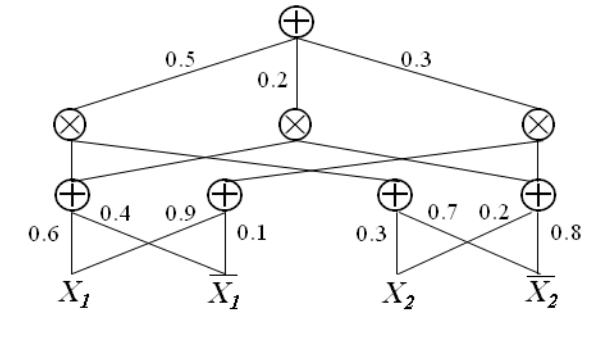}
  \caption{An example of an SPN representing a naive Bayes mixture, over variables $X_1,X_2$, taken from \cite{6130310}}
  \label{spn}
\end{figure}
SPNs are rooted directed graphical models that provide for an efficient way of representing the network polynomial \cite{Darwiche2000ADA} of a BN \cite{6130310}, as a multilinear function $\sum_{\textbf{x}}f(\textbf{x})\prod_{n=1}^{N} \mathbbm{1}_{x_n}$. An example of such an SPN can be seen in Figure \ref{spn}, taken from \cite{6130310}. Here $f(\cdot)$ is the  (possibly unormalized) probability distribution of the BN, $\textbf{x}$ is a vector containing all the variables of the model, i.e.,  $x_1,\cdots,x_N$, the summation is over all possible states, and $\mathbbm{1}_{x_n}$ is the indicator function. An SPN $\cal S$ over Boolean variables $x_1,\cdots,x_N$ has leaves corresponding to indicators $\mathbbm{1}_{x_1},\cdots,\mathbbm{1}_{x_n}$ and $\mathbbm{1}_{\bar x_1},\cdots,\mathbbm{1}_{\bar x_n}$
and whose internal nodes are sums and products.

Any edge exiting a sum node has a non-negative weight assigned to it. The value of a product node is the product of its children, while the value of a sum node is a weighted sum of its children, $\sum_{u_j \in Ch(u_i)} w_{ij}{\cal S}_{j}(\textbf{x})$, where $Ch(u_i)$ is the set containing the children of node $u_i$, and ${\cal S}_{j}$ is the sub-SPN rooted at node $u_j$. SPNs can represent a wide class of models, including weighted mixtures of univariate distributions; see  \cite{6130310} for discussions.
\subsection{Causality}
Causal inference is an approach where, apart from probabilistic information, extra information about the mechanism governing the variables' interactions are encoded into the model. This allows reasoning about more complex queries, such as interventions and counterfactuals \cite{Pearl:2009:CMR:1642718}. These can be seen as extending standard probabilities with the ability to infer what happens if a variable is forced to attain a value, by an external intervention, or what would happen had a variable obtained a different value from the one it obtained in the actual world. 

The usual setting is to represent the set of probabilistic dependencies through a BN,  but on top of that encode the specific mechanism that determines the value of each variable, too. In this sense, it is more general than just having a BN, since we not only possess a distribution over the variables, but also a set of equations. In what follows we denote by $\V$ the set of variables that are internal to the model, and by $\U$ the exogenous or external variables (that act as random, latent, factors). We use $\pazocal{R}$ to denote the set containing the plausible values of each variable. Every endogenous (internal) variable is assigned an equation determining its value as a function of both its endogenous and exogenous parents in the BN, called structural equation. 
 

\begin{definition}\label{defn:causalmodel}  A causal model $\pazocal{M}$ is a pair $(\pazocal{S},\pazocal{F} )$ where $\pazocal{S}$ is a signature $(\U, \V,\pazocal{R})$ and $\pazocal{F}$ 
is a set of structural equations $ \{\pazocal{F}_V : V \in \V \}.$
\end{definition}

An interesting remark is that, although the structural equations are essential for the specification of the model, it turns out that once you have a fully specified probabilistic distribution, it is possible to answer interventional queries without possessing the functional equations \cite{Pearl2009CausalII}. There are various ways to achieve that, such as utilizing the \textit{rules of do-calculus} \cite{Pearl:2009:CMR:1642718}, but we will go for a different approach, using minimal information about the structure of the underlying BN. We are going to utilize the following formula to compute the effect of intervening on a variable, $A$, on the rest of the model's variables, $\textbf{X}_{-A}$ \cite{Pearl:2009:CMR:1642718}:
\begin{align*}
\Pr(\textbf{X}_{-A}| do(A=\alpha)) = \frac{\Pr(\textbf{X}_{-A},A=\alpha)}{\Pr(A=\alpha | pa_{A})}
\end{align*}
where $pa_{A}$ denotes the set of A's parents. Looking at this equation we see that the only thing we need to specify is $pa_{A}$. This has the advantage of requiring only local information, so it is not necessary to specify the full BN, but only the variables that have an effect on the intervened variable. We will return to this observation in the following sections.

\subsection{Optimization}
\begin{figure}[t]
  \centering
    \includegraphics[scale=0.9]{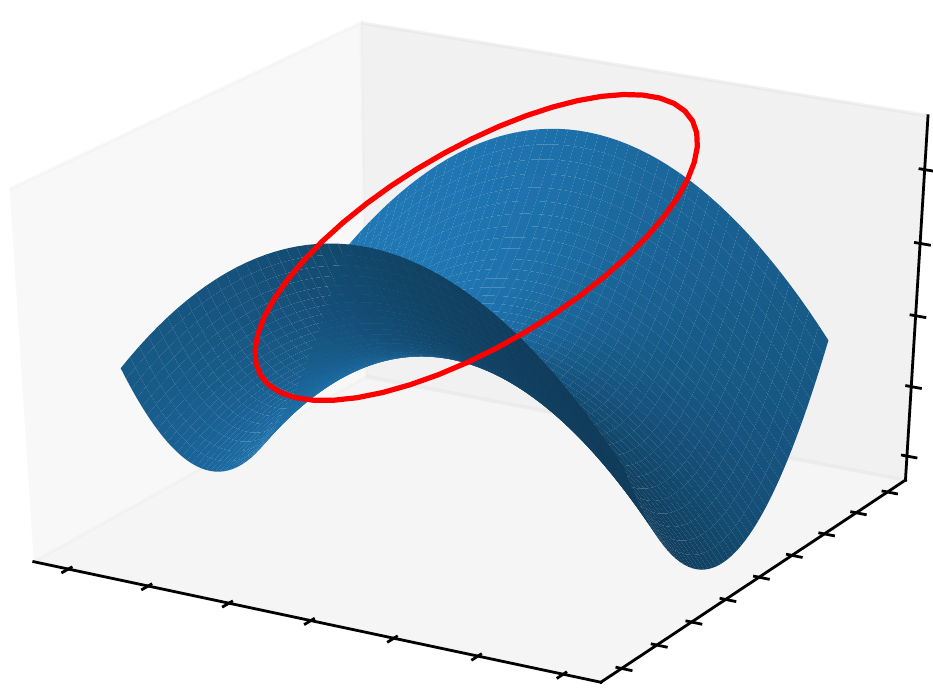}
  \caption{An example of optimizing a function, while constraining the solution to lie in an ellipse}
  \label{opt}
\end{figure}
Constrained optimization is a discipline concerned with developing techniques allowing for optimizing functions under a set of constraints. For example, Figure \ref{opt} depicts the problem of minimizing a function, while requiring the solution to belong in an ellipse.  One of the most common ways to address that, is to transform the objective function, so it takes the constraints into account. The problem of interest is to maximize the likelihood of a model (with a vector of parameters \textbf{w}), $L(\textbf{w})$ under constraints $C_i(\textbf{w})=0,~1\leq i \leq N$, so:
\begin{align*}
&max_{\textbf{w}} L(\textbf{w}) \\
&s.t.~ C_1(\textbf{w})=0\\
&~~~~~~~~~ \cdots \\
&~~~~~~ C_N(\textbf{w})=0
\end{align*}
The transformed objective function, $\Lambda$, introduces a number of auxiliary variables, as many as the constraints, $\lambda_1,\cdots,\lambda_N$, and takes the following form $\Lambda(\textbf{w},\lambda_1,\cdots,\lambda_N) = L(\textbf{w}) + \sum_{n=1}^N  \lambda_n C_n(\textbf{w}) $. It can be shown that all of the solutions of the original problem correspond to stationary points of the new objective function \cite{lagrange}. 

There are various numerical methods to solve this problem, such as projected gradient descent, where an initial vector $\textbf{w}^{(0)}$ is updated incrementally, and then gets projected onto the surface defined by the constraints, until it converges to a solution of the problem. Furthermore, in cases where the objective function is in a special form, such as a quadratic polynomial, other approaches might be more efficient. See \cite{MrquezNeila2017ImposingHC} for a more extensive discussion on the subject.

Optimization problems like the above require all of the feasible solutions to satisfy the constraints. These constraints are referred to as \textit{hard}. Alternative formulations of the problem could yield feasible solutions not satisfying the constraints. These constraints are called \textit{soft}, because instead of demanding the solutions to adhere to them, we introduce a penalty term in the objective function, for each time they get violated. For example, if all of the $C_i(\textbf{w})=0,~1\leq i \leq N$ were treated as soft constraints, then after setting  $\lambda_1,\cdots,\lambda_N$ to some value reflecting the cost of violating the corresponding constraint, the soft version of the problem would be to maximize the function $L(\textbf{w}) + \sum_{n=1}^N \lambda_n C_n(\textbf{w}) $, so each time some $C_i$ is not equal to zero, it induces a penalty. In this case, all $\lambda_i$ are treated as hyperparameters, so they are specified before the optimization takes place. Furthermore, now we are interested in the maxima of this function, as opposed to the case of hard constraints, where we were interested in the stationary points of the transformed function. This problem could be solved using a wide range of well-known optimization algorithms, such as Adam \cite{Adam}.
\section{Main Results}
The majority of contemporary machine learning models rely on maximum likelihood (ML) estimation for setting the values of their parameters. The approaches we discussed earlier transform the optimization objective, enhancing the resulting model with additional properties. One limitation, in such a setting, is that the constraints are expressed in terms of the parameters, directly. However, in most models it is not clear how probabilistic relationships can be expressed in term of the parameters, making it difficult to utilize the existing approaches in order to achieve our goal.  

Our approach is motivated from such formulations and builds on the following idea: since the majority of machine learning models are differentiable and utilize ML estimation, if we could find a class of models where it is feasible to uncover a correspondence between parameters and probabilities, then we could use constrained optimization approaches, in order to equip the model with additional properties. 

As we will see in what follows, most probabilistic constraints are expressed, by definition, as an equality between probabilities. For example, if we want to incorporate the assumption that "$A$ is independent of $B$", we have to ensure that the equality $\Pr(A,B)=\Pr(A)\Pr(B)$ holds in the trained model (we will only present equality constraints, but following our strategy one could incorporate inequality constraints as well).
\subsection{Conditional constraints}
We will start with presenting the case of constraining the likelihood so it enforces the various conditional distributions of some variables to be equal. Formally, assume a variable $Y$, a categorical variable $A$, and a set of variables, $\bf X$. We are interested in modelling the joint distribution of these variables, but we would also like to incorporate some background knowledge into the model, specifically we would like it to satisfy the condition $\Pr(Y| A=\alpha, \textbf{X}) = \Pr(Y| A=\alpha',\textbf{X})$, where we assume that $A$ is a binary variable, in order to make the presentation easier to follow. In this equation we do not explicitly specify the values of the variables in $\bf X$, rather we want the condition to hold regardless of their specific instantiation. We could also be interested in constraints of the form $\Pr(Y| A=\alpha) = \Pr(Y| A=\alpha')$, where in this case we do not condition on \textbf{$X$}. Constraints similar to this, appear in the fair AI literature \cite{Zemel:2013:LFR:3042817.3042973,article3,Hardt:2016:EOS:3157382.3157469,GrgiHlaa2016TheCF}, where the objective is to eliminate bias, such as racial discrimination, from predictive models, by enforcing an appropriate set of conditions. 


An additional remark about the flexibility of expressing constraints in this form can be seen when considering context-specific properties. In the above formulation we left the values of $\bf X$ unspecified, but there might be cases where it is known that some properties hold only when some of the remaining variables acquire specific values. 
To take such information into account we should just adapt the constrain so some of the variables in \textbf{$X$} are set to their corresponding values.

As we have stated above, we are going to use SPNs to model the data, due to their provable tractability and applicability in a wide range of problems and the fact that a connection between probabilistic queries and the model's parameters can be established. As we will see, this is crucial for our approach, in the sense that in the general case it is not clear how to achieve this connection, but the polynomial representation of SPNs allow us to uncover it and use it to train such a model under a set of probabilistic constraints. 


The following results establishes the relationship between probabilistic constraints and the parameters of an SPN, $\textbf{w}$.  
\begin{theorem}
Let $\cal S$ be an SPN representing the joint distribution of variables $X_1,\cdots,X_n$. Let $X_i,X_j$ be two binary variables, then the constraint $\Pr(X_i|X_j=0) = \Pr(X_i|X_j=1)$ is equivalent to a multivariate linear system of two equations on the SPN's parameters.
\end{theorem}

\begin{proof}
Let ${\cal S}(\textbf{x})= \sum_{\textbf{x}} f(\textbf{x})\prod_{n=1}^{N} \mathbbm{1}_{x_n}$ 
be the network polynomial of an SPN. The equality $\Pr(X_i|X_j=1)= \Pr(X_i| X_j=0)$ can be rewritten as follow: 
\begin{align}
& \Pr(X_i|X_j=1)= \Pr(X_i|X_j=0) \implies \frac{ \Pr(X_i,X_j=1)}{\Pr(X_j=1)}= \frac{\Pr(X_i,X_j=0)}{\Pr(X_j=0)}\\ \nonumber
&\implies \Pr(X_i,X_j=1) \cdot \Pr( X_j=0) = \Pr(X_i,X_j=0) \cdot \Pr(X_j=1)
\end{align}
Next, we express the above probabilities in terms of $\cal S$ (where $X$ corresponds to the assignment $X=1$, and $\neg X$ to $X=0$):
\begin{align*}
&\Pr(X_i,X_j=1) = \sum_{\textbf{x}:x_i,x_j}f(\textbf{x})\mathbbm{1}_{x_i} + \sum_{\textbf{x}:\neg x_i,x_j}f(\textbf{x})\mathbbm{1}_{\neg x_i}\\
&\Pr(X_i,X_j=0) = \sum_{\textbf{x}:x_i,\neg x_j}f(\textbf{x})\mathbbm{1}_{x_i} + \sum_{\textbf{x}:\neg x_i, \neg x_j}f(\textbf{x})\mathbbm{1}_{\neg x_i}\\
&\Pr(X_j=1) = \sum_{\textbf{x}:x_j}f(\textbf{x})\\
&\Pr(X_j=0) = \sum_{\textbf{x}:\neg x_j}f(\textbf{x})
\end{align*} 
We now substitute these equation to (1) to get that:
\begin{align*}
& \sum_{\textbf{x}:\neg x_j}f(\textbf{x}) \cdot \sum_{\textbf{x}:x_i,x_j}f(\textbf{x})\mathbbm{1}_{x_i} + \sum_{\textbf{x}:\neg x_j}f(\textbf{x}) \cdot \sum_{\textbf{x}:\neg x_i,x_j}f(\textbf{x})\mathbbm{1}_{\neg x_i}=\\
& \sum_{\textbf{x}:x_j}f(\textbf{x}) \cdot \sum_{\textbf{x}:x_i,\neg x_j}f(\textbf{x})\mathbbm{1}_{x_i} + \sum_{\textbf{x}:x_j}f(\textbf{x}) \cdot \sum_{\textbf{x}:\neg x_i, \neg x_j}f(\textbf{x})\mathbbm{1}_{\neg x_i}
\end{align*}
This is an equality between polynomials, meaning that the coefficients must be equal, so:
\begin{align*}
& \sum_{\textbf{x}:\neg x_j}f(\textbf{x}) \cdot \sum_{\textbf{x}:x_i,x_j}f(\textbf{x}) = \sum_{\textbf{x}:x_j}f(\textbf{x}) \cdot \sum_{\textbf{x}:x_i,\neg x_j}f(\textbf{x})\\
&  \sum_{\textbf{x}:\neg x_j}f(\textbf{x}) \cdot \sum_{\textbf{x}:\neg x_i,x_j}f(\textbf{x}) = \sum_{\textbf{x}:x_j}f(\textbf{x}) \cdot \sum_{\textbf{x}:\neg x_i, \neg x_j}f(\textbf{x})
\end{align*}
These constraints are expressed in terms of the model's parameters and they are, clearly, multivariate polynomials, specifically linear ones, since, in each equations, there are two products, so if we look, for example, at the ones in the first equation,  $\sum_{\textbf{x}:x_i,x_j}f(\textbf{x}) \cdot \sum_{\textbf{x}:\neg x_j}f(\textbf{x})$ and $\sum_{\textbf{x}:x_i,\neg x_j}f(\textbf{x}) \cdot \sum_{\textbf{x}:x_j}f(\textbf{x})$, the terms that appear in one factor don't appear on the other one, since the summation is performed over disjoint sets. \qed
\end{proof}

\subsection{Interventional constraints}
A more complex class of distributions, used extensively in causal modelling \cite{Pearl2009CausalII}, are interventional ones. They represent the probability of a variable after an external intervention on another variable. It is not always possible to estimate them using observational distributions, but when assuming that all of the model's variables are observed, then it is possible to express the interventional distribution in terms of the observational one \cite{Pearl:2009:CMR:1642718}. For the rest of this section we will make the closed-world assumption, meaning that there are no unobserved confounders between the variables. 

The new objective is to train a model while incorporating constraints of the form $\Pr(\textbf{X}_{-A}| do(A=\alpha)) = \Pr(\textbf{X}_{-A}|do( A=\alpha'))$, where $\textbf{X}_{-A}$ denotes the set of all the model's variables, excluding $A$. Constraints of this kind have powerful implications regarding the causal mechanisms between $A$ and the rest of the variables. This could be seen clearly when considering similar constraints to the one above, such as $\Pr(\textbf{X}_{-A}| do(A=\alpha)) = \Pr(\textbf{X}_{-A})$, which means that setting $A$ to a certain value does not influence the distribution of the rest of the variables. Intuitively, this means that $A$ has no causal influence on any of the remaining variables. 

As we have mentioned in a previous section, we will base our approach on a well known formula connecting the interventional to the observational distribution \cite{Pearl:2009:CMR:1642718}:
\begin{align*}
\Pr(\textbf{X}_{-A}| do(A=\alpha)) = \frac{\Pr(\textbf{X}_{-A},A=\alpha)}{\Pr(A=\alpha | pa_{A})}
\end{align*}

Depending on the application, it is possible there is enough background knowledge available to specify $pa_A$. There might be other applications though, where this is not an option, due to the complexity of the problem or insufficient a priori information. In these cases, methods from the field of \textit{feature selection} could be utilized. The aim of these methods is to identify the Markov Blanket of a set of variables, so it is closely related to specifying the parents of a variable. Conditioning on the Markov Blanket, instead of just the parents, can serve as an approximation of the desired distribution, so there is a wide range of methods \cite{Zhang:2011:KCI:3020548.3020641,PetBuhMei15,Zheng:2018:DNT:3327546.3327618} for performing this step. Furthermore, there are some existing approaches that under some assumptions recover just the parent set of a variable \cite{PetBuhMei15}, so these could be employed, instead. Assuming we possess the parents of the variable of interest, we can show the following:\\
\begin{theorem}
Let $\cal$ S be an SPN representing the joint distribution of variables $X_1,\cdots,X_n$. Let $X_i$ be a binary variable, then the constraint $\Pr(\textbf{X}_{-i}|do(X_i=0)) = \Pr(\textbf{X}_{-i}|do(X_i=1))$ is equivalent to a multivariate linear system of equations on the SPN's parameters.
\end{theorem}

\begin{proof}
We will prove this, following the same reasoning as in the previous proof, so we first need to rewrite the given constraint:
\begin{align*}
& \Pr(\textbf{X}_{-i}|do(X_i=0)) = \Pr(\textbf{X}_{-i}|do(X_i=1))\\
&\Rightarrow \frac{\Pr(\textbf{X}_{-i},X_i=0)}{\Pr(X_i=0 | pa_{X_i})} = \frac{\Pr(\textbf{X}_{-i},X_i=1)}{\Pr(X_i=1 | pa_{X_i})}\\ \nonumber
& \Rightarrow \Pr(\textbf{X}_{-i},X_i=0) \cdot \Pr(X_i=1 | pa_{X_i}) = \Pr(\textbf{X}_{-i},X_i=1) \cdot \Pr(X_i=0 | pa_{X_i})\\ \nonumber
& \Rightarrow  \Pr(\textbf{X}_{-i},X_i=0) \cdot \Pr(X_i=1 , pa_{X_i}) = \Pr(\textbf{X}_{-i},X_i=1) \cdot \Pr(X_i=0 , pa_{X_i})
\end{align*}
The next step is to express these probabilities in terms of the network polynomial and substitute them to the above expression. Since these computations are lengthy and routine, we will not present them here. The important observation is that it is not difficult to see that we end up with a system of multivariate polynomials, in this case, too. To prove they are linear ones as well, it suffices to note that in both products $\Pr(\textbf{X}_{-i},X_i=0) \cdot \Pr(X_i=1 , pa_{X_i})$ and $\Pr(\textbf{X}_{-i},X_i=1) \cdot \Pr(X_i=0 , pa_{X_i})$, the set of parameters involved in the first factor is disjoint with the one appearing in the second factor, since the parameters that remain after setting $X_i=0$ vanish when setting $X_i=1$ (and vice versa). \qed 
\end{proof}

\subsection{Independence constraints}
The last kind of constraints we will present are those enforcing independence between variables. One of the most common forms of available background knowledge is expressed in terms of probabilistic independence, including conditional independence. There are some already existing approaches, like the ones presented in earlier sections, allowing for incorporating rules expressed as propositional formulas within the model, but doing the same with probabilistic ones still poses a major challenge.

An approach can be shown to facilitate the above task, without requiring us to provide new models, but by just transforming the optimization objective. As we discussed earlier, the definition of independence itself could be utilized to express the corresponding constraint. 

We should also note that it is possible to incorporate conditional independence or context specific information within the model, too, using the exact same method. Although similar, since usually both of them relies on conditioning, each one provides different insights about the problem at hand. So, for example, conditional constraints could be of the form: if we know the value of a variable, $Z$, then $A$ and $B$ are independent. On the other hand, context specific independence is stronger, since it might state that only when $Z=z$ we know that $A$ and $B$ are independent. However, it is not difficult to see that each of these independencies can be expressed as $\Pr(A,B|Z) = \Pr(A|Z)\Pr(B|Z)$ and $\Pr(A,B|Z=z)=\Pr(A|Z=z)\Pr(B|Z=z)$, respectively.

Assuming, in this case as well, that the objective is to train an SPN satisfying constraints like the above, we can show that it amounts to optimizing a function over a set of multivariate quadratic polynomial constraints.
\begin{theorem}
Let $\cal S$ be an SPN representing the joint distribution of variables $X_1,\cdots,X_n$. Let $X_i,X_j$ be two binary variables, then the constraint $\Pr(X_i,X_j) = \Pr(X_i) \cdot \Pr(X_j)$ is equivalent to a multivariate quadratic system of four equations on the SPN's parameters.
\end{theorem}

\begin{proof}
To prove this result it is not necessary to rewrite the given constraint, so we can start with expressing these probabilities in terms of $\cal S$:
\begin{align*}
&\Pr(X_i,X_j) = \sum_{\textbf{x}:x_i,x_j}f(\textbf{x})\mathbbm{1}_{x_i}\mathbbm{1}_{x_j} + \sum_{\textbf{x}:\neg x_i,x_j}f(\textbf{x})\mathbbm{1}_{\neg x_i}\mathbbm{1}_{x_j} \\
&~~~~~~~~~~~~~~~~+  \sum_{\textbf{x}:x_i,\neg x_j}f(\textbf{x})\mathbbm{1}_{x_i}\mathbbm{1}_{\neg x_j} + \sum_{\textbf{x}:\neg x_i,\neg x_j}f(\textbf{x})\mathbbm{1}_{\neg x_i}\mathbbm{1}_{\neg x_j}\\
&\Pr(X_i)= \sum_{\textbf{x}:x_i}f(\textbf{x})\mathbbm{1}_{x_i} + \sum_{\textbf{x}:\neg x_i}f(\textbf{x})\mathbbm{1}_{\neg x_i}\\
&\Pr(X_j)= \sum_{\textbf{x}:x_j}f(\textbf{x})\mathbbm{1}_{x_j} + \sum_{\textbf{x}:\neg x_j}f(\textbf{x})\mathbbm{1}_{\neg x_j}
\end{align*}
Next, we substitute these quantities to the constraint's equation, so we get that:
\begin{align*}
&\sum_{\textbf{x}:x_i,x_j}f(\textbf{x})\mathbbm{1}_{x_i}\mathbbm{1}_{x_j} + \sum_{\textbf{x}:\neg x_i,x_j}f(\textbf{x})\mathbbm{1}_{\neg x_i}\mathbbm{1}_{x_j} + \sum_{\textbf{x}:x_i,\neg x_j}f(\textbf{x})\mathbbm{1}_{x_i}\mathbbm{1}_{\neg x_j} \\
&+ \sum_{\textbf{x}:\neg x_i,\neg x_j}f(\textbf{x})\mathbbm{1}_{\neg x_i}\mathbbm{1}_{\neg x_j}
 = \sum_{\textbf{x}:x_i}f(\textbf{x}) \cdot \sum_{\textbf{x}:x_j}f(\textbf{x})\mathbbm{1}_{x_i} \mathbbm{1}_{x_j}\\
 &+ \sum_{\textbf{x}:x_i}f(\textbf{x}) \cdot \sum_{\textbf{x}:\neg x_j}f(\textbf{x}) \mathbbm{1}_{x_i} \mathbbm{1}_{\neg x_j}
 + \sum_{\textbf{x}:\neg x_i}f(\textbf{x}) \cdot \sum_{\textbf{x}:x_j}f(\textbf{x}) \mathbbm{1}_{\neg x_i} \mathbbm{1}_{x_j}\\
 &+ \sum_{\textbf{x}:\neg x_i}f(\textbf{x}) \cdot \sum_{\textbf{x}:\neg x_j}f(\textbf{x})\mathbbm{1}_{\neg x_i} \mathbbm{1}_{\neg x_j}
\end{align*}
Equating the coefficients we get the following system of equations:
\begin{align*}
& \sum_{\textbf{x}:x_i,x_j}f(\textbf{x}) = \sum_{\textbf{x}:x_i}f(\textbf{x}) \cdot \sum_{\textbf{x}:x_j}f(\textbf{x}) \\
& \sum_{\textbf{x}:\neg x_i,x_j}f(\textbf{x}) = \sum_{\textbf{x}:\neg x_i}f(\textbf{x}) \cdot \sum_{\textbf{x}:x_j}f(\textbf{x})\\
& \sum_{\textbf{x}:x_i,\neg x_j}f(\textbf{x}) = \sum_{\textbf{x}:x_i}f(\textbf{x}) \cdot \sum_{\textbf{x}:\neg x_j}f(\textbf{x}) \\
& \sum_{\textbf{x}:\neg x_i,\neg x_j}f(\textbf{x}) = \sum_{\textbf{x}:\neg x_i}f(\textbf{x}) \cdot \sum_{\textbf{x}:\neg x_j}f(\textbf{x})
\end{align*}
Each of these equations correspond to a multivariate polynomial, as in all the previous cases, but this time they are quadratic, instead. This is because, in each equation, the sums appearing on the right hand side have some terms in common. For example, looking at the first equation, the assignment setting all the variables equal to $1$ is compatible with both summations, so the term $f(x_1,\cdots,x_n)$ appears in both of them. Clearly, by multiplying them we end up with a squared parameter. \qed
\end{proof}

\section{Applying the framework}
In this section we will demonstrate how to derive the system of equations that correspond to a single constraint. Let's assume we would like to train an SPN, $\cal S$, over three binary variables, $X_1,X_2,X_3$, satisfying the property that $X_1$ and $X_2$ are independent. The canonical polynomial of $\cal S$ \cite{Darwiche:2003:DAI:765568.765570} is:
\begin{align*}
&{\cal S}(X_1,X_2,X_3,\neg X_1,\neg X_2, \neg X_3) = \theta_1 X_1X_2X_3 + \theta_2 \neg X_1X_2X_3 + \theta_3 X_1\neg X_2X_3\\  & + \theta_4 \neg X_1\neg X_2X_3
+ \theta_5 X_1\neg X_2\neg X_3 + \theta_6 \neg X_1X_2\neg X_3 + \theta_7 X_1X_2\neg X_3\\ & + \theta_8 \neg X_1\neg X_2\neg X_3
\end{align*}
where each $\theta_i$ is equal to the probability of the specific configuration of $X_1,X_2,X_3$ following it, so, for example, in the term $\theta_5 X_1\neg X_2\neg X_3$, $\theta_5 =\Pr( X_1,\neg X_2,\neg X_3)$

The joint probability of, say, $X_1,X_2$ is given by the above polynomial, after substituting both $X_3,\neg X_3$ by $1$. Furthermore, the probability of $X_1$ is given after substituting all of $X_2,\neg X_2, X_3,\neg X_3$ by $1$, whereas substituting $X_1,\neg X_1, X_3,\neg X_3$ by $1$, gives the probability of $X_2$.

At this point, it is time to utilize the condition we would like to enforce, $\Pr(X_1,X_2)=\Pr(X_1)\Pr(X_2)$. Substituting these probabilities by the corresponding polynomial, as discussed in the previous paragraph, yields the following:
\begin{align*}
&(\theta_1 + \theta_7) X_1 X_2  + (\theta_3 + \theta_5) X_1 \neg X_2 + (\theta_2 + \theta_6) \neg X_1 X_2 + (\theta_4 + \theta_8) \neg X_1 \neg X_2 \\ &  = 
  (\theta_1 + \theta_3 + \theta_5 + \theta_7)\cdot ( \theta_1 + \theta_2 + \theta_6 + \theta_7) X_1 X_2\\ & + (\theta_1 + \theta_3 + \theta_5 + \theta_7)\cdot ( \theta_3 + \theta_4 + \theta_5 + \theta_8) X_1 \neg X_2\\
& + (\theta_2 + \theta_4 + \theta_6 + \theta_8)\cdot ( \theta_1 + \theta_2 + \theta_6 + \theta_7) \neg X_1 X_2\\ & + (\theta_2 + \theta_4 + \theta_6 + \theta_8)\cdot ( \theta_3 + \theta_4 + \theta_5 + \theta_8) \neg X_1 \neg X_2
\end{align*}

This can be seen as an equivalence between polynomials, so all the coefficients must be equal, meaning that:
\begin{align*}
&\theta_1 + \theta_7 = (\theta_1 + \theta_3 + \theta_5 + \theta_7)\cdot ( \theta_1 + \theta_2 + \theta_6 + \theta_7)\\
&\theta_3 + \theta_5 = (\theta_1 + \theta_3 + \theta_5 + \theta_7)\cdot ( \theta_3 + \theta_4 + \theta_5 + \theta_8) \\
&\theta_2 + \theta_6 = (\theta_2 + \theta_4 + \theta_6 + \theta_8)\cdot ( \theta_1 + \theta_2 + \theta_6 + \theta_7)\\
& \theta_4 + \theta_8 = (\theta_2 + \theta_4 + \theta_6 + \theta_8)\cdot ( \theta_3 + \theta_4 + \theta_5 + \theta_8)
\end{align*}

At this point, since each $\theta_i$ has probabilistic semantics, we perform a sanity check, by replacing them with the corresponding probability they represent and rewrite the whole system in terms of probabilities. This will provide some insights on the underlying constraints, as well as some hints on alternative ways to incorporate the constraints in the model during optimization. 
\begin{align*}
&\theta_1 + \theta_7 = \Pr(X_1,X_2, X_3) + \Pr(X_1,X_2, \neg X_3) = \Pr(X_1,X_2)\\
&\theta_2 + \theta_6 = \Pr(\neg X_1,X_2, X_3) + \Pr(\neg X_1,X_2,\neg X_3) = \Pr(\neg X_1,X_2)\\
& \theta_3 + \theta_5 = \Pr(X_1,\neg X_2, X_3) + \Pr(X_1,\neg X_2, \neg X_3) = \Pr(X_1,\neg X_2)\\
& \theta_4 + \theta_8 = \Pr(\neg X_1,\neg X_2, X_3) + \Pr(\neg X_1,\neg X_2,\neg X_3) = \Pr(\neg X_1,\neg X_2)\\
&\theta_1 + \theta_3 + \theta_5 + \theta_7  = \Pr(X_1,X_2, X_3) + \Pr(X_1,\neg X_2, X_3)\\
& ~~~~~~~~~~~~~~~~~~~~~~~~~~~+ \Pr(X_1,\neg X_2, \neg X_3) + \Pr(X_1,X_2, \neg X_3)\\
& ~~~~~~~~~~~~~~~~~~~~~~~~~~~ = \Pr(X_1)\\
\end{align*}
\begin{align*}
& \theta_1 + \theta_2 + \theta_6 + \theta_7 = \Pr(X_1,X_2, X_3) + \Pr(\neg X_1,X_2, X_3) \\
& ~~~~~~~~~~~~~~~~~~~~~~~~~~~ + \Pr(\neg X_1,X_2,\neg X_3) + \Pr(X_1,X_2, \neg X_3) \\
&~~~~~~~~~~~~~~~~~~~~~~~~~~~ = \Pr(X_2)\\
& \theta_3 + \theta_4 + \theta_5 + \theta_8 = \Pr(X_1,\neg X_2, X_3) + \Pr(\neg X_1,\neg X_2, X_3)\\
& ~~~~~~~~~~~~~~~~~~~~~~~~~~~+ Pr(X_1,\neg X_2, \neg X_3) + Pr(\neg X_1,\neg X_2, \neg X_3)\\
&~~~~~~~~~~~~~~~~~~~~~~~~~~~= \Pr(\neg X_2)\\
& \theta_2 + \theta_4 + \theta_6 + \theta_8 = \Pr(\neg X_1,X_2, X_3) + \Pr(\neg X_1,\neg X_2, X_3) \\
& ~~~~~~~~~~~~~~~~~~~~~~~~~~~+ \Pr(\neg X_1,X_2,\neg X_3) + \Pr(\neg X_1,\neg X_2,\neg X_3)\\
&~~~~~~~~~~~~~~~~~~~~~~~~~~~= \Pr(\neg X_1)
\end{align*}

Substituting all these quantities to the original system, we get the following:
\begin{align*}
&\Pr(X_1,X_2) = \Pr(X_1) \cdot \Pr(X_2)\\
&\Pr(X_1,\neg X_2) = \Pr(X_1) \cdot \Pr(\neg X_2)\\
&\Pr(\neg X_1,X_2) = \Pr(\neg X_1) \cdot \Pr(X_2)\\
&\Pr(\neg X_1,\neg X_2) = \Pr(\neg X_1) \cdot \Pr(\neg X_2)
\end{align*}

These are exactly the conditions that have to hold for two binary variables to be independent. 
Having these equations, the training of the model can go on, interpreting them  as either hard or soft constraints. If they are incorporated as soft constraints, new terms are added in the objective function, but since all of them are differentiable, any standard optimization algorithm could be utilized to train the model. In contrast, if they are treated as hard constraints, projected gradient descend or approaches like the one developed in \cite{MrquezNeila2017ImposingHC} would need to be used to train the SPN. We would also like to note that by utilizing the probabilistic representation of the system, we do not have to do computations between the $\theta_i$'s, explicitly, but we can instead utilize the SPN to get, for example, the quantity $\Pr(X_1,X_2) = \Pr(X_1) \cdot \Pr(X_2)$, and then in order to obtain the gradient, we just use the sub-SPNs corresponding to the distributions of $\{X_1,X_2\},\{X_1\},\{X_2\}$, as well as the rules of differentiation. In the case of the term $\Pr(X_1,X_2) - \Pr(X_1) \cdot \Pr(X_2)$, this would just mean that we would have to compute 
\begin{align*}
\frac{\partial \Pr(X_1,X_2)}{\partial \textbf{w}} - \frac{\partial \Pr(X_1)}{\partial \textbf{w}}\cdot \Pr(X_2) - \Pr(X_1)\cdot \frac{\partial \Pr(X_2)}{\partial \textbf{w}}
\end{align*}
which can be computed by using the corresponding sub-SPNs.

\section{An extension to PSDDs}
\begin{figure}[t]
  \centering
    \includegraphics[scale=0.35]{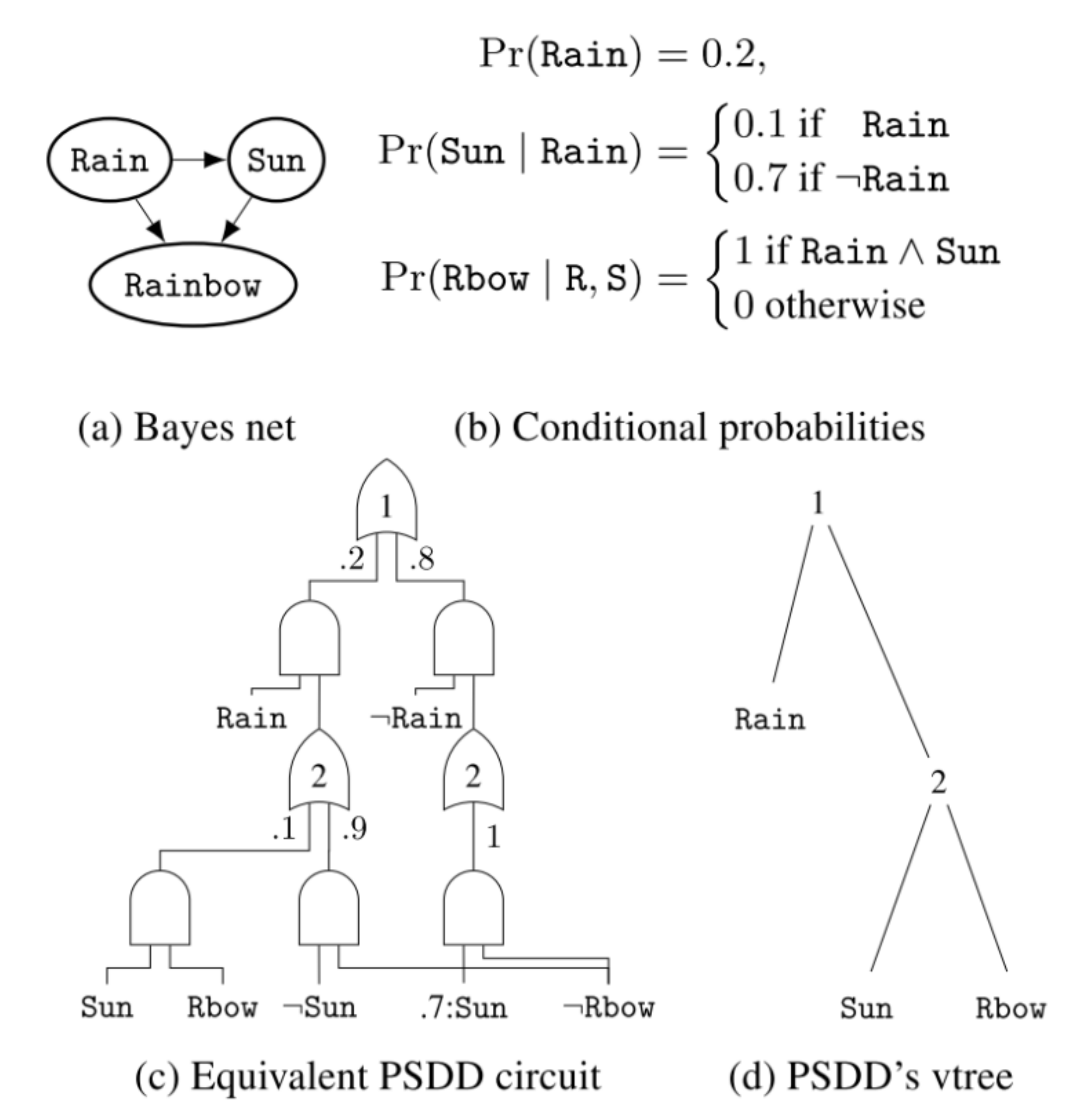}
  \caption{An example of a BN and the corresponding PSDD, taken from \cite{KR148005}}
  \label{psdd}
\end{figure}
So far, we have based our presentation solely on SPNs, since they allow for directly connecting their polynomial representation to probabilistic expressions. In this section we are going to argue that PSDDs allow for training under constraints as well. Briefly, PSDDs are a probabilistic extension of sentential decision diagrams (SDDs) \cite{inproceedings20}. SDDs are used in order to represent a propositional logic theory, while PSDDs utilize this representation and recursively define a probability distribution over it. This procedure results to a rooted directed graph, where terminal nodes can be either a literal, $\top$, or $\bot$, while decision (intermediate) nodes are of the form $(p_1 \wedge s_1)\vee \dots \vee (p_k \wedge s_k)$, where the $p_i's$ are called primes and the $s_i's$ subs. The primes form a partition, meaning they are mutually exclusive and their disjunction is valid. Each prime $p_i$ in a decision node  is assigned a non-negative parameter $\theta_i$ such that $\sum_{i=1}^k \theta_i=1$ and $\theta_i=0$ if and only if $s_i=\bot$. Additionally each terminal node corresponding to $\top$ has a parameter $\theta$ such that $0 < \theta < 1$. Figure \ref{psdd} is an example of a BN, along with the probability distribution defined over it, and a PSDD capturing this distribution, taken from \cite{KR148005}.

PSDDs do not come with a compact polynomial representation, like SPNs, but, still, can be trained under probabilistic constraints. In the previous section we saw how to derive an equivalent way of expressing the resulting system of equations in a probabilistic form. Furthermore, we argued how this form can be utilized in order to train the model, as an alternative to the original formulation. The only requirement is to be able to compute the corresponding probabilities in an efficient way, so it is feasible to infer these quantities at each training iteration, as well as to be differentiable functions w.r.t the parameters, $\textbf{w}$. PSDDs satisfy both requirements \cite{KR148005}, meaning that the same framework, as in section 4, using the probabilistic formulation, can be applied in order to incorporate constraints during training. In addition, considering that PSDDs can also incorporate prior information in the form of propositional expressions, by training them under our proposed framework they can now incorporate both probabilistic and propositional information.

The above demonstrate the ability of any tractable probabilistic model to be trained under probabilistic constraints, as long as it is differentiable, since the computation of marginal or conditional distributions are efficient, by definition. In turn, our approach, although initially based on SPNs, can be seen as model-agnostic, when it is applied to tractable probabilistic models.

\section{Discussion and Conclusions}

In the previous sections we presented an approach allowing to train SPNs under probabilistic constraints. SPNs are tractable models, meaning that probabilistic inference is efficient, since marginal or conditional queries can be computed in time linear in its size. This is an appealing property, because otherwise additional steps, such as MCMC sampling, would be necessary in order to perform inference. Taking that into account, SPNs can not only take probabilistic assumptions into account, but they can also easily compute such queries. 

In our opinion, an other interesting point is that our work could be seen as related to the work that has been done in the field of \textit{Fairness in AI}, but from a generative modelling point of view. The main objective in the field is to formalize criteria leading to fair predictions, and train models satisfying these criteria. For example, enforcing a condition such as $\Pr(\hat{y}=1|a=0) = \Pr(\hat{y}=1|a=1)$, where $a$ is a protected binary attribute and $\hat{y}$ is the model's prediction, has been proposed \cite{Zemel:2013:LFR:3042817.3042973}. In our setting there is no predicted variable, so this condition cannot be applied. However, an analogous condition could be utilized when dealing with generative modelling, such as $\Pr(y=1|a=0) = \Pr(y=1|a=1)$.

Another recent approach on incorporating constraints is through the introduction of the semantic loss function \cite{pmlr-v80-xu18h}. In this sense, we consider it related to our work, but the existing framework allows for constraints over the predicted variable, so this is not immediately applicable to generative models. In addition, another significant difference between this and our approach is that the semantic loss function can only express rules in propositional logic, where our method can make use of probabilistic rules.

In this work we provided a way to equip SPNs with background information. This adds to the growing literature on constraints and machine learning that is emerging recently. The key difference in our method is that it is made for generative models, unlike the majority of the existing work, as well as it exhibits how the model's intrinsic architecture can be utilized to do so, allowing us to recover a system of equations. We hope the results of this paper will lead to a new range of applications making use of tractable generative models that allow the incorporation of non-trivial logical and probabilistic prior knowledge. 

\bibliographystyle{ecai}
\bibliography{tractablebib}

\end{document}